\documentclass[10pt, conference, letterpaper]{IEEEtran}
\IEEEoverridecommandlockouts
\usepackage[algo2e]{algorithm2e} 
\makeatletter
\newcommand{\removelatexerror}{\let\@latex@error\@gobble}
\makeatother
\usepackage[utf8]{inputenc} 
\usepackage[T1]{fontenc}    
\usepackage{url}            
\usepackage{booktabs}       
\usepackage{amsfonts}       
\usepackage{nicefrac}       
\usepackage{microtype}      
\usepackage{xcolor}         
\usepackage{subfigure}
\usepackage{enumitem}
\usepackage{cite}
\usepackage{amsthm}
\usepackage{wrapfig}
\usepackage{float}
\usepackage{amsmath,amssymb}
\usepackage{multirow}
\usepackage{graphicx}
\usepackage{textcomp}
\usepackage{tabularx}
\usepackage{pifont}
\usepackage[english]{babel}
\usepackage{algorithm}
\usepackage{algpseudocode}
\usepackage{tikz}
\usepackage{colortbl}
\definecolor{mygray}{gray}{.9}
\usepackage[font=small, figurename=Fig.]{caption} 
\renewcommand{\figurename}{Fig.}

\newcommand*\circled[1]{\tikz[baseline=(char.base)]{
            \node[shape=circle,draw,inner sep=0.2pt] (char) {#1};}}
\let\oldnl\nl
\newcommand{\nonl}{\renewcommand{\nl}{\let\nl\oldnl}}
\newtheorem{definition}{Definition}[section]

\newtheorem{theorem}{Theorem}

\def\BibTeX{{\rm B\kern-.05em{\sc i\kern-.025em b}\kern-.08em
    T\kern-.1667em\lower.7ex\hbox{E}\kern-.125emX}}
\begin{document}

\title{Upcycling Noise for Federated Unlearning\\
}

\author{Jianan Chen, Qin Hu, Fangtian Zhong, Yan Zhuang, Minghui Xu}
\maketitle

\begin{abstract}
In Federated Learning (FL), 
multiple clients collaboratively train a model without sharing raw data. This paradigm can be further enhanced by Differential Privacy (DP) to protect local data from information inference attacks and is thus termed DPFL. An emerging privacy requirement, ``the right to be forgotten'' for clients, poses new challenges to DPFL but remains largely unexplored. 
Despite numerous studies on federated unlearning (FU), they are inapplicable to DPFL because the noise introduced by the DP mechanism compromises their effectiveness and efficiency. In this paper, we propose \textit{Federated Unlearning with Indistinguishability (FUI)} to unlearn the local data of a target client in DPFL for the first time. FUI consists of two main steps: \textit{local model retraction} and \textit{global noise calibration}, resulting in an unlearning model that is statistically indistinguishable from the retrained model. Specifically, we demonstrate that the noise added in DPFL can endow the unlearning model with a certain level of indistinguishability after local model retraction, and then fortify the degree of unlearning through global noise calibration. 
Additionally, for the efficient and consistent implementation of the proposed FUI, we formulate a two-stage Stackelberg game to derive optimal unlearning strategies for both the server and the target client.
Privacy and convergence analyses confirm theoretical guarantees, while experimental results based on four real-world datasets illustrate that our proposed FUI achieves superior model performance and higher efficiency compared to mainstream FU schemes. Simulation results further verify the optimality of the derived unlearning strategies.
\end{abstract}

\section{Introduction}
Federated Learning (FL) is a distributed machine learning framework that facilitates the collaborative training of models among multiple clients, without sharing their local data, under the coordination of a centralized server \cite{mcmahan2017communication}. While this distributed approach inherently guarantees a certain degree of privacy by uploading the local models instead of the original data, the risk of potential privacy breaches still exists as sensitive information can be inferred from the shared models~\cite{geiping2020inverting,luo2021feature,li2022auditing}. To mitigate this risk, Differential Privacy (DP) is integrated into FL to provide a theoretical privacy guarantee by adding noise to local models so that they would not leak information about training samples, which can be termed DP-enhanced FL (DPFL) \cite{geyer2017differentially,wei2020federated}. 

Despite the robust privacy protection provided by DPFL, there is an emerging privacy requirement concerning the case of any client who would like to quit the FL process. This issue is technically known as federated unlearning (FU) to guarantee ``the right to be forgotten'' \cite{regulation2016regulation} for clients. 
Prominent regulations, such as the European Union’s General Data Protection Regulation (GDPR) \cite{voigt2017eu}, 
also mandate the deletion of personal data from models upon request. While DPFL effectively reduces the risk of privacy leakage by adding noise to local models, it cannot inherently facilitate the unlearning of local data. The challenges of unlearning in DPFL have never been touched, presenting a significant gap in current research.

According to whether remaining clients are required to actively participate in unlearning, existing research on FU could be divided into two categories: public resource consuming schemes~\cite{liu2021federaser,zhao2023federated,yuan2023federated,liu2022right,halimi2022federated,cao2023fedrecover, ding2023incentive} and stakeholder engaging schemes~\cite{wu2022federated,wang2022federated,xia2023fedme,zhang2023fedrecovery,alam2023get,gong2021bayesian}. For the first category of schemes, the computational and time costs are usually significant, and the remaining clients lack the motivation to participate in the unlearning process. While for stakeholder engaging schemes, the noise introduced in DPFL would compromise their effectiveness and efficiency. In other words, the noise of the DP mechanism leads to significant performance degradation of the unlearning model. Due to the above reasons, all existing FU methods cannot be directly applied to DPFL. 

In this work, we propose an unlearning scheme, \textit{Federated Unlearning with Indistinguishability (FUI)}, tailored for DPFL by upcycling the noise added into local models for partial unlearning and then deriving a statistically indistinguishable global model compared to the retrained model for full unlearning. Specifically,
FUI consists of two steps: \textit{local model retraction} and \textit{global noise calibration}. In local model retraction, the target client reverses the FL process by maximizing the loss function based on the Limited-memory Broyden–Fletcher–Goldfarb–Shanno (L-BFGS) algorithm \cite{liu1989limited} to improve the efficiency of unlearning, where the noise injected in DPFL turns out to achieve a certain degree of unlearning. In global noise calibration, the unlearning result is investigated to determine whether additional noise is needed to achieve $\epsilon$-indistinguishability.
FUI only requires the participation of the server and the target client requesting to be forgotten, minimizing the overall resource consumption 
and improving the time efficiency. 
To facilitate the implementation of FUI, we also resolve the interest conflict between the server and the target client via formulating a Stackelberg game with the server as the leader and the client as the follower. Finally, the optimal unlearning strategies, including both the penalty factor $p$ for the server and the privacy requirement $\epsilon$ for the target client, are derived through backward induction.

The contributions of this paper are manifold:
\begin{itemize}
\item We are the first to study the issue of federated unlearning in DPFL. The proposed FUI approach only requires the participation of the server and the target client but not other clients staying in DPFL, making the DPFL system energy-efficient and sustainable. 
\item Our scheme delicately upcycles the noise injected by the DP mechanism to achieve a certain degree of unlearning by conducting local model retraction at the target client.
\item We use the concept of indistinguishability to theoretically quantify the degree of unlearning and enhance the indistinguishability of the unlearning model through global noise calibration, calculating the noise gap and injecting additional noise when necessary.
\item We formulate a two-stage Stackelberg game to derive optimal unlearning strategies toward the efficient and consistent implementation of the proposed FUI.
\item We theoretically prove that our proposed FUI can effectively generate an unlearning model, which is $\epsilon$-indistinguishable from the retrained model. The FUI algorithm is also theoretically convergent.\item Experimental results demonstrate that our proposed FUI scheme can achieve a balance between the performance of the unlearning model and the privacy protection for clients in an efficient manner. Meanwhile, the derived optimal unlearning strategies realize the stably high utilities for both the server and the target client.
\end{itemize}

The rest of the paper is organized as follows. Section \ref{sec:related} reviews the related work in the fields of DPFL and federated unlearning. Section \ref{sec:system} outlines the system model of DPFL and Section \ref{sec:fui} delves into the detailed design of FUI. Section~\ref{sec:game} presents the Stackelberg game for implementing FUI. 
Section \ref{sec:theoretical} discusses the theoretical analysis of our approach. Section \ref{sec:experimental} illustrates the experimental setup and analyzes the results, while Section \ref{sec:conclusion} concludes the whole paper with the discussion of future directions.

\section{Related Work}\label{sec:related}
Researchers have extensively explored the integration of DP and FL. 
McMahan et al. formally define DP-based FL with theoretical privacy guarantees as a pioneering study \cite{mcmahan2017learning}. Truex et al. develop the LDP-Fed system with selection and filtering algorithms for perturbing and sharing partial parameter updates with the server \cite{truex2020ldp}. Girgis et al. achieve a better balance between communication efficiency and local DP by deploying the CLDP-SGD algorithm \cite{girgis2021shuffled}. 
The application of DPFL has extended to many practical and cutting-edge areas, including the Internet of Vehicles \cite{olowononi2021federated,zhao2020local}, medical information~\cite{choudhury2019differential,adnan2022federated}, and quantum platforms \cite{rofougaran2024federated}. 

To guarantee ``the right to be forgotten'' \cite{regulation2016regulation} for participants, FU in vanilla FL has received considerable attention recently. The state-of-the-art FU schemes can be divided into two categories, \textit{public resource consuming} and \textit{stakeholder engaging} schemes, according to whether other staying clients besides the target client need to be actively involved in the unlearning process or not. 
In the first category of studies, Liu et al present FedEraser, as the first FU algorithm, by trading the server’s storage for the remaining clients' time to reconstruct the unlearning model \cite{liu2021federaser}. Halimi et al. obtain the final unlearning result by performing local unlearning on the target client and then using the model to initialize the retrain at all clients \cite{halimi2022federated}. To support unlearning in federated recommendation systems, Yuan et al. propose Federated Recommendation Unlearning (FRU) inspired by the transaction rollback mechanism, which removes user's contributions by rolling back and calibrating historical updates, and then accelerates retraining \cite{yuan2023federated}. Ding et al. propose a four-stage game to formulate the interaction in the learning and unlearning process, guiding the incentive design. The proposed incentive mechanism allows the server to retain valuable clients with less cost \cite{ding2023incentive}. Cao et al. apply the unlearning technique against poisoning attacks by proposing FedRecover, which uses the historical information and the computation of the remaining clients to restore the attacked model \cite{cao2023fedrecover}. However, since these methods require the involvement of remaining clients, the computational and time costs are significant, and the remaining clients are less motivated to join this process.

The other category only requires the participation of the target client working with the centralized server. 
Wu et al. propose a novel FU approach that eliminates clients' contributions by subtracting accumulated historical update information and utilizes knowledge distillation to recover model \cite{wu2022federated}. Wang et al. introduce the concept of Term Frequency Inverse Document Frequency (TF-IDF) to quantize the class discrimination of channel and unlearning information of specific class in the model by pruning and fine-tuning \cite{wang2022federated}. To solve the problem of mobile network instability in digital twins, Xia et al. propose the framework of FedME2, which achieves more accurate unlearning through a memory evaluation module and multi-loss training approach \cite{xia2023fedme}. Zhang et al. propose FedRecovery, which realizes FU under the definition of indistinguishability by injecting noise into the unlearning model \cite{zhang2023fedrecovery}. Unfortunately, none of the above methods can be directly applied to DPFL because the noise introduced by the DP mechanism in the learning process can prominently sacrifice their advantages in unlearning, including the performance of the unlearning model, as well as the effectiveness and efficiency of unlearning. 

Driven by this challenge, we propose to leverage the inherent characteristics of the DP mechanism by recycling the noise in the DPFL process to facilitate the unlearning of a target client. A follow-up procedure quantifies the degree of unlearning to determine the amount of additional noise regarding achieving the full unlearning goal with a theoretical privacy guarantee.  

\section{System Model}\label{sec:system}



As mentioned in the Introduction, although FL keeps clients' raw data locally, it is still vulnerable to inference attacks on uploaded local models. DP, therefore, is employed to fortify FL by injecting noise into local models, theoretically ensuring a certain level of privacy protection for clients' data. This fusion of DP with FL, termed DPFL, empowers clients to engage in FL without worrying about data privacy, thereby expanding the reach of FL in sectors where confidentiality is paramount. 

\begin{figure}
\centering
\includegraphics[scale=0.82]{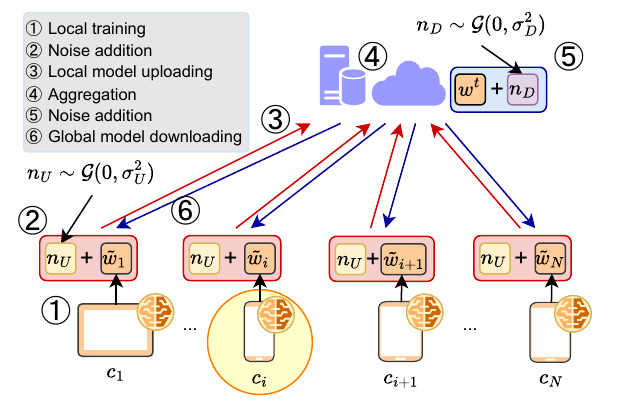}
\caption{Illustration of DPFL.}
\label{fig:FLDP}
\end{figure}

The framework\footnote{Since there are various implementations of DPFL, we consider a widely accepted one defined by \cite{wei2020federated} in this paper.} of DPFL is demonstrated in Fig. \ref{fig:FLDP}, 
where a group of clients $\mathcal{C}=\{c_1,\dots,c_i,\dots,c_N\}$ collaboratively train a model under the coordination of a cloud server. 
The main procedures in DPFL can be summarized below: 
\begin{enumerate}[label=\protect\circled{\arabic*}]
\item After the server initializes the global model, each client $c_i\in\mathcal{C}$ downloads it for conducting local model training using its own dataset $\mathcal{D}_i=\{(x_j,y_j)\}^{|\mathcal{D}i|}_{j=1}$, where $x_j$ is the $j$-th input sample, $y_j$ is the corresponding label, and $|\mathcal{D}_i|$ is the number of training samples in $\mathcal{D}_i$.
\item Each client $c_i$ aims to find a vector $\tilde{w}_i$, i.e., the local model, to minimize the loss function $F_i(\tilde{w}_i)$, where $\tilde{w}_i$ will be injected with Gaussian noise $n_U\sim\mathcal{G}(0,\sigma_{U}^2)$ to protect the sensitive information from being inferred, with $\sigma_U$ indicating the scale of the noise.
\item Then client $c_i$ submits the protected model $w_i=\tilde{w}_i+n_U$ to the server.
\item In $t$-th global round, the server aggregates all submissions from clients to update the global model $w^t=\sum_{i=1}^{N}\frac{|\mathcal{D}_i|}{\sum_{i=1}^{N} |\mathcal{D}_i|} w_i^t$. 
\item The server further protects this aggregated model by adding noise $n_D\sim\mathcal{G}(0,\sigma_{D}^2)$ with $\sigma_{D}$ being the noise scale.
\item The server broadcasts the protected aggregated model $w^t+n_D$ to each client for the next round of training.
\end{enumerate}
After sufficient iterations, the global model converges to 
$w^*=\arg \min_{w} \sum_{i=1}^{N}\frac{|\mathcal{D}_i|}{\sum_{i=1}^{N} |\mathcal{D}_i|} F_i(w)$. 

In this process, steps \circled{2} and \circled{5} eliminate the privacy leakage risk of local data by adding noise to realize the privacy protection requirement defined below: 
\begin{definition}
($\eta$-DPFL). An FL mechanism $\mathcal{M}:\mathcal{X} \xrightarrow{}\mathcal{R}$, with domain $\mathcal{X}$ and range $\mathcal{R}$, satisfies $\eta$-DP for each client $c_i$ if for any pair of neighboring datasets $\mathcal{D},\mathcal{D}'\in\mathcal{X}$, and for all sets $O\subseteq\mathcal{R}$ of possible models, we have:
\begin{equation*}
    Pr[\mathcal{M}(\mathcal{D})\in O]\leq e^{\eta}\cdot Pr[\mathcal{M}(\mathcal{D}')\in O],
\end{equation*}
where $\eta>0$ is the privacy parameter and usually small.
\end{definition}
Specifically, $\eta$-DP can be achieved by adding Gaussian noise $n_U\sim\mathcal{G}(0,\sigma_{U}^2)$ 
to the local updates for the uplink channel and $n_D\sim\mathcal{G}(0,\sigma_{D}^2)$ to the global model for the downlink channel~\cite{wei2020federated}.  
To ensure that the given noise distribution $n_U\sim\mathcal{G}(0,\sigma_{U}^2)$ preserves $\eta$-DP in uplink channel, the noise scale $\sigma_U= \frac{\Delta s_U}{\eta}$, where $\Delta s_U=\frac{2C}{m}$, with $C$ being the clipping threshold for bounding local model $w_i$ and $m=\min\{|\mathcal{D}_i|\}$. As for the downlink channel, extra noise $n_{D}\sim \mathcal{G}(0,{\sigma_D}^2)$ would be added to the aggregated model $w$ by the server so as to reduce the chance of information leakage due to the large number of model aggregations. And thus $\sigma_D$ is calculated by:
\begin{align*}
    \sigma_D=\begin{cases}
         \frac{2C(T^2-L^2N)}{mN\eta}& T> L\sqrt{N}, \\
         0 & T\leq L\sqrt{N}.
    \end{cases}
\end{align*}
$T$ is the number of aggregation times, $L$ is the number of exposures of local models, and $N$ is the total number of clients. 

In DPFL, a target client, denoted as $c_i$ (highlighted with a yellow circle in Fig. \ref{fig:FLDP}), is granted the authority for security or privacy considerations to request withdrawing the contribution of its local data, which is known as \textit{federated unlearning (FU)}. This withdrawal requires the server to eliminate $\mathcal{D}_i$'s influence on the global model, thereby upholding the ``right to be forgotten'' for each client. While a straightforward approach of FU involves all remaining clients retraining their local models to obtain a new global model $w^{RE}$, this method is resource-intensive and time-consuming. Consequently, the server is tasked with identifying an alternative algorithm capable of deriving an unlearning model $w^{UN}$, which could match the performance of $w^{RE}$ without compromising the legitimate right of the target client while meeting $\eta$-DP requirement. 
In spite of various state-of-the-art FU schemes, none of them resolves this challenge of unlearning in DPFL.

\section{Federated Unlearning with Indistinguishability (FUI)}\label{sec:fui}
FU is generally challenging. As for DPFL, the addition of DP noise makes the performance loss of the existing FU schemes even greater. DPFL urgently requires an unlearning solution that can embrace noise. Thus, we propose FUI as an unlearning scheme tailored for DPFL. Unlike existing FU schemes, FUI exploits the natural properties of DP to implement the unlearning process in DPFL for obtaining an unlearning model $w^{UN}$ with a similar performance as the retrained model $w^{RE}$. Technically, we can resort to the concept of $\epsilon$-indistinguishability to characterize the similarity between $w^{UN}$ and $w^{RE}$. 

\begin{definition}
($\epsilon$-indistinguishability). For random variables $X$ and $Y$ over domain $\mathcal{X}$, $X$ and $Y$ are $\epsilon$-indistinguishable if for all possible subsets of outcomes $\mathcal{Z}\subseteq \mathcal{X}$, there exist 
\begin{align}\label{eq:def_indis}
    Pr[X\in \mathcal{Z}]\leq e^{\epsilon}\cdot Pr[Y\in \mathcal{Z}],\nonumber\\
    Pr[Y\in \mathcal{Z}]\leq e^{\epsilon}\cdot Pr[X\in \mathcal{Z}],
\end{align}
where 
$\epsilon>0$ is a small real number.
\end{definition}
This definition provides a statistical guarantee that two random variables are indistinguishable under this criterion. Then the similarity between the unlearning model and the retrained model can be quantified by this statistical indistinguishability. 
Now we could reformulate the FU task in DPFL as below:
\begin{definition}
(FU in $\eta$-DPFL). In $\eta$-DPFL, for the unlearning request of a target client $c_i$, the server needs to remove the influence of $\mathcal{D}_i$ on the global model $w$ by finding an unlearning model $w^{UN}$ which is $\epsilon$-indistinguishable with the retrained model $w^{RE}$ while meeting the $\eta$-DP requirement. 
\end{definition}

Then we propose FUI to derive the desired unlearning model by two steps: \textit{local model retraction} and \textit{global noise calibration}. Local model retraction allows the client to autonomously forget the knowledge of the local data and obtain a preliminary unlearning model $w^{LR}$, and global noise calibration further calculates the necessary noise to be added to $w^{LR}$ to derive $w^{UN}$ which satisfies $\epsilon$-indistinguishability with $w^{RE}$.

\subsection{Local Model Retraction}
To unlearn local data of $c_i$, an intuitive idea is to reverse the local learning process, that is, training the model parameters to maximize the loss function instead of minimizing it. While previous studies have employed various methods, such as gradient ascent and projected gradient ascent, to achieve this reversal, they are not compatible with DP, potentially compromising their effectiveness and performance due to noise interference. Moreover, these methods are usually time-consuming, thus reducing the efficiency of unlearning.

In our proposed FUI, we leverage the Limited-memory Broyden–Fletcher–Goldfarb–Shanno (L-BFGS) algorithm \cite{liu1989limited} which optimizes functions more efficiently than traditional gradient ascent methods, particularly in scenarios involving numerous variables. L-BFGS approximates the inverse Hessian matrix using limited memory, which facilitates fast convergence and reduces the requirement for storing extensive data. It can also handle high-dimensional problems efficiently and adjust step sizes automatically, enhancing its utility for large-scale optimization tasks and thus being suitable for FU.

In particular, target client $c_i$ computes a reference model $w_{ref}=\frac{1}{N-1}(Nw^t-w_i^{t})$, where $w^t$ is the global model after $t$ rounds of aggregation and $w_i^{t}$ denotes the submission of $c_i$ in the $t$-th round. Then, $c_i$ applies L-BFGS algorithm to solve an optimization problem:
\begin{equation*}
  w^{LR}=\arg \max_{w\in\{v\in\mathcal{R}:||v-w_{ref}||_2\leq \delta\}} F_i(w),
\end{equation*}
where $\delta$ is the radius of the $l_2$-norm ball,  determining the range of retraction.
Specifically, for a given step size $\alpha>0$ and the approximation of the inverse Hessian matrix $\mathbf{H}$, client $c_i$ iterates $w$ until convergence:
\begin{align*}
&w_0 =w_{ref},\\
&\mathbf{H}_0=\lambda\mathbf{I},\\
&
\begin{cases}
S_{k}=w_{k+1}-w_{k},\\
y_{k}=F_i(w_{k+1})-F_i(w_{k}),\\
\mathbf{H}_{k+1}=(\mathbf{I}-\frac{S_{k}y_{k}^T}{y_{k}^T S_{k}})\mathbf{H}_{k}(\mathbf{I}-\frac{y_{k}S_{k}^T}{y_{k}^T S_{k}})+\frac{S_{k}S_{k}^T}{y_{k}^T S_{k}},\\
w_{k+1}=w_k+\alpha \mathbf{H}_k \nabla F_i(w_k),
\end{cases}
\end{align*}
where $\mathbf{H}_k$ is the approximation to the inverse Hessian matrix at iteration $k$, $\lambda$ is a scalar parameter commonly set as $1$, and $\nabla F_i$ is the gradient of $F_i$. 

Once the server receives $w^{LR}$ from $c_i$, it can be found that local model retraction naturally achieves a certain level of indistinguishability because of the noise added by the DP mechanism, which is summarized in the following theorem. 
\begin{theorem}\label{Theo:1}
If a randomized mechanism $\mathcal{M}$ satisfies $\eta$-DP, then the outgoing random variables $X$ and $Y$ of $\mathcal{M}$ are $\frac{\eta^2}{2}$-indistinguishable.
\end{theorem}
\begin{proof}
We first define the Rényi divergence of order $\alpha$ between $X$ and $Y$ as $D_{\alpha}(X(x)||Y(x))=\frac{1}{\alpha-1}\ln{[\mathbb{E}_{x\sim X}(\frac{X=x}{Y=x})^{\alpha-1}]}$, where $\mathbb{E}_{x\sim X}(X=x)$ represents the expectation of variable $X$, and $x$ follows the distribution of $X$.

According to \cite{bun2016concentrated}, if a mechanism $\mathcal{M}$ satisfies $\eta$-DP, it implies $D_{\alpha}(\mathcal{M}(x)||\mathcal{M}'(x))\leq D_{\infty}(\mathcal{M}(x)||\mathcal{M}'(x))\leq\eta$ for $\alpha\in(1,\infty]$.
For $X,Y,\mathcal{M}(x)\in O$,
by the definition of Rényi divergence, we have $D_\alpha(Y(x)||Y'(x)) = \frac{1}{\alpha-1}\ln\mathbb{E}_{x\sim Y}\left(\frac{Pr[Y=x]}{Pr[Y'=x]}\right)^{\alpha-1}$.
Since $D_\alpha(Y(x)||Y'(x))\leq\eta$, we can get
$\frac{1}{\alpha-1}\ln\mathbb{E}_{x\sim Y}\left(\frac{Pr[Y=x]}{Pr[Y'=x]}\right)^{\alpha-1} \leq \eta$, which derives $\mathbb{E}_{x\sim Y}\left(\frac{Pr[Y=x]}{Pr[Y'=x]}\right)^{\alpha-1} = \sum_{x\in O}(\frac{Pr[Y=x]}{Pr[Y'=x]})^{\alpha-1}\leq e^{(\alpha-1)\eta}$.
Hence, we have $\sum_{x\in O}(\frac{Pr[Y=x]}{Pr[Y'=x]})^{\eta}\leq e^{\eta^2}$. Then we have
\begin{align*}
&Pr[Y\in O]\\
&=\sum_{x\in O}Pr[Y=x]\left(\frac{Pr[X=x]}{Pr[X=x]}\right)^{\frac{\eta}{2}}\left(\frac{Pr[X'=x]}{Pr[X'=x]}\right)^{\frac{\eta}{2}}\\
&= \sum_{x\in O}Pr[Y=x]\left(\frac{Pr[X'=x]}{Pr[X=x]}\right)^{\frac{\eta}{2}}\left(\frac{Pr[X=x]}{Pr[X'=x]}\right)^{\frac{\eta}{2}}\\
&\leq e^{\frac{\eta^2}{4}}\cdot e^{\frac{\eta^2}{4}}\sum_{x\in O}Pr[X=x]
= e^{\frac{\eta^2}{2}} Pr[X\in O].
\end{align*}
Therefore, we obtain $Pr[Y\in O] \leq e^{\frac{\eta^2}{2}}Pr[X\in O]$. Similarly, we can prove
$Pr[X\in O] \leq e^{\frac{\eta^2}{2}}Pr[Y\in O]$. 
Thus, $X$ and $Y$ satisfy ${\frac{\eta^2}{2}}$-indistinguishability.
\end{proof}

Nevertheless, local model retraction only achieves an indistinguishability level of ${\frac{\eta^2}{2}}$ between $w^{LR}$ and $w^{RE}$, which could be inadequate compared to the expected $\epsilon$-indistinguishability. 
To address this concern, our FUI scheme incorporates a crucial step, global noise calibration, to reinforce $w^{LR}$ towards meeting $\epsilon$-indistinguishability.

\subsection{Global Noise Calibration}
In this step, the server investigates whether the received $w^{LR}$ satisfies $\epsilon$-indistinguishability and adds appropriate noise as a calibration. Global noise calibration meticulously adjusts the noise levels added to the unlearning model, compensating for the deficiency noted in the local model retraction phase. This step is vital to ensure that the differential privacy guarantees align closely with the desired $\epsilon$-indistinguishability.

Notice that FU's goal is $w^{UN}$ and $w^{RE}$ being $\epsilon$-indistinguishable. According to Theorem \ref{Theo:1}, we have already achieved $\frac{\eta^2}{2}$-indistinguishability by $w^{UN}=w^{LR}$. Then we can respectively calculate the required noise scale $\Tilde{\sigma_1}$ for $\frac{\eta^2}{2}$-indistinguishability and $\Tilde{\sigma_2}$ for $\epsilon$-indistinguishability as:
$\Tilde{\sigma_1}=\frac{\sqrt{2}d}{2\eta},    \Tilde{\sigma_2}=\frac{d}{\sqrt{\epsilon}},$ where $d$ is the upper bound of the distance between model parameters, calculated by (36) in \cite{zhang2023fedrecovery}. Here we define the noise gap $g=\Tilde{\sigma_1}-\Tilde{\sigma_2}$. If $g\geq 0$, we believe that the existing $\eta$-DP mechanism has already guaranteed the unlearning requirement. If $g <0 $, we will add additional noise to $w^{LR}$ so as to realize $\epsilon$-indistinguishability. In particular, the additional noise scale could be calculated as $\sigma_{cali}=\sqrt{\Tilde{\sigma_2}^2-\Tilde{\sigma_1}^2}$ \cite{hyvarinen1998independent}. With the additional calibrated noise $n_{cali}\sim \mathcal{G}(0,\sigma_{cali}^2)$, the final unlearning model could be calculated as:
\begin{equation}
    w^{UN}=w^{LR}+n_{cali}.
\end{equation}

For clarity, we summarize FUI in Algorithm \ref{algo}. 
In local model retraction, we leverage the L-BFGS algorithm to achieve efficient client-side unlearning, which avoids the hassle of most FU schemes that require the involvement of the remaining clients and greatly improves efficiency. To begin with, client $c_i$ calculates the reference model $w_{ref}$ (Line 1). Then $c_i$ conducts the L-BFGS algorithm to iteratively derive $w^{LR}$ until the increment of $w_k$ is less than the convergence threshold $\tau$ (Lines 2-11).
Next, the server calculates the noise gap $g$ to determine whether extra noise is needed (Lines 12-14). If yes, noise $n_{cali}$ will be integrated into $w^{LR}$ to derive $w^{UN}$ (Lines~15-19); otherwise, $w^{LR}$ will be returned as $w^{UN}$ (Lines 20-22).
\begin{algorithm}
\caption{FUI}
\label{algo}
\LinesNumbered
\raggedright
\SetNlSty{textbf}{}{}  
\SetAlgoNlRelativeSize{-3}  
\textbf{Input:} Number of clients $N$, target client $c_i$, number of aggregation rounds $t$, privacy parameters $\eta$ and $\epsilon$, step size $\alpha$, the approximation of the inverse Hessian matrix $\textbf{H}$, convergence threshold $\tau$\\
\textbf{Output:} Unlearning model $w^{UN}$
\removelatexerror
\begin{algorithm2e}[H]
\nonl \textnormal{\textbf{Local Model Retraction at Client $c_i$:}}\\
Compute reference model $w_{ref} \gets \frac{1}{N-1}(N w^t - w^{t}_i)$\\
$w_0 \gets w_{ref}$, $k \gets 0$\\
\While{$||w_k-w_{ref}||_2\leq\delta$}{
    $w_{k+1} \gets w_k + \alpha \textbf{H}_k \nabla F_i(w_k)$\\
    \If{$||w_{k+1}-w_k||_2\leq\tau$}{
         $w^{LR} \gets w_{k+1}$\\
        \textnormal{\textbf{break}};}
    $k\gets k+1$\\
}
\Return $w^{LR}$ \textnormal{to server}\\
\nonl \textnormal{\textbf{Global Noise Calibration at the Server:}}\\
$\Tilde{\sigma_1}\gets\frac{\sqrt{2}d}{2\eta}$\\
$\Tilde{\sigma_2}\gets\frac{d}{\sqrt{\epsilon}}$\\
$g\gets \Tilde{\sigma_1}-\Tilde{\sigma_2}$\\
\If{$g<0$}{
         $\sigma_{cali}\gets\sqrt{\Tilde{\sigma_2}^2-\Tilde{\sigma_1}^2}$\\
        \textnormal{Sample noise $n_{cali}$ from Gaussian distribution} $n_{cali}\sim\mathcal{G}(0,\sigma_{cali})$\\
        \Return $w^{UN}\gets w^{LR} + n_{cali}$}
        \Else
            {\Return $w^{UN}\gets w^{LR}$}
\end{algorithm2e}
\end{algorithm}

\section{Optimal Unlearning Strategies in FUI}\label{sec:game}
Though the proposed FUI enables the server to efficiently unlearn the data of client $c_i$ via achieving an $\epsilon$-indistinguishable $w^{UN}$, it remains unclear how to determine the optimal value of $\epsilon$. In fact, from the perspective of the server, as the model owner, unlearning is never a desirable operation since the performance of the global model can degrade after unlearning some training data; 
however, from the perspective of the client, unlearning is to obtain an unlearning model $w^{UN}$ as similar to the retrained model $w^{RE}$ as possible, which implies that the smaller the value of $\epsilon$ in $\epsilon$-indistinguishability, the better the unlearning degree. A conflict thus exists, where the target client's strict requirement on indistinguishability can reduce the performance of the global model, damaging the interest of the server. 
To derive the best unlearning implementation, we formulate the interactions of FUI between the server and the target client as a Stackelberg game. Then analyzing the Nash equilibrium can help the client determine the optimal unlearning parameter $\epsilon$ under the constraint of the server's policy. 

\subsection{Game Formulation}
Since the server will suffer from reduced model performance in FU, a penalty function $P(p,\epsilon)$ is designed to discourage the withdrawal of clients: 
\begin{equation*}
    P(p,\epsilon)=\frac{p}{{\epsilon^2+1}},
\end{equation*}
where $p\in(0, p_{max}]$ is the penalty factor determined by the server and $\epsilon$ is the unlearning parameter of the target client.
According to \cite{wu2020value}, the model performance is inversely proportional to the square of the privacy requirement, where a smaller $\epsilon$ can degrade the performance of the global model more, and thus the server charges a higher penalty.

According to \cite{ding2023strategic}, the model performance is also proportional to the logarithm of the training data size, 
so we denote the loss of model performance due to unlearning as $Q(\epsilon)$ which can be calculated by:
\begin{equation*}
    Q(\epsilon)=\frac{a\ln{(|D_{-i}|^2)}}{b\epsilon^2+1},
\end{equation*}
where $|D_{-i}|$ is the total dataset size of all the rest clients except the target client $c_i$, $a>0$ is a constant indicating the impact of data size on model performance, and $b>0$ is a constant reflecting the effect of the privacy parameter. 

Then we can draw the utility of the server $U_s$ as:
\begin{equation*}
U_s(p)=-Q(\epsilon)+P(p,\epsilon)-\Psi_s,
\end{equation*}
where $\Psi_s>0$ is a constant representing the cost of executing FUI for the server.

As for the target client, 
the benefit of privacy protection after unlearning is denoted by $R$, which can be defined as: 
\begin{equation*}
R(\epsilon)=\frac{r}{s\epsilon^2+1}+l,
\end{equation*}
where $\epsilon\in[\epsilon_{min},\frac{\eta^2}{2})$ is the unlearning requirement determined by the client, and $r,s,l>0$ are constants. Note that the lower limit $\epsilon_{min}$ is to prevent the extreme unlearning requirement from the client, which may lead to the global model full of noise and thus useless for the server; the upper limit $\frac{\eta^2}{2}$ comes from the indistinguishability naturally achieved by $\eta$-DP as discussed in Theorem 1. In other words, when the unlearning requirement from the client is too loose, there is no need to specify it since $\eta$-DPFL has already guaranteed a better level of unlearning. As mentioned earlier, the smaller the value of $\epsilon$, the more the global model forgets the local data, bringing the higher benefits of privacy preservation for the client.

Thus, the utility of the target client $U_c$ can be calculated as:
\begin{equation*}
U_c(\epsilon)=R(\epsilon)-P(p,\epsilon)-\Psi_c,
\end{equation*}
where 
$\Psi_c>0$ is the cost for $c_i$ to run FUI.

Now, the interactions between the server and the target client $c_i$ towards FUI can be formulated as a two-stage Stackelberg game, where the server is the leader and the target client is the follower, defined below:
\begin{itemize}
\item \textbf{Stage 1}: The server 
determines the penalty factor $p$ by maximizing its utility $U_s$. 
\item \textbf{Stage 2}: Once receiving the penalty factor $p$, the target client $c_i$ decides the unlearning requirement $\epsilon$ via optimizing its utility $U_c$. 
\end{itemize}

\subsection{Optimal Strategies}
To solve the above two-stage Stackelberg game, we transfer it into two separate optimization problems that are resolved sequentially. Specifically, we employ backward induction to analyze the optimal strategy of Stage 2 first and then derive the optimal strategy of Stage 1.

In Stage 2, the goal of client $c_i$ is to maximize its utility:
\begin{equation*}\label{opt_1}
\begin{split}
\textbf{Problem 1:}~&\epsilon^*=\arg\max{U_c(\epsilon)},\\
&\,s.t.
\begin{cases}
\epsilon_{min} \leq \epsilon < \frac{\eta^2}{2}, \\
U_c\geq0.\\
\end{cases}
\end{split}
\end{equation*}
The first constraint is the range of variable $\epsilon$ and the second constraint comes from the motivation of the target client requesting unlearning. 
Problem 1 is a nonlinear optimization problem, and we could solve it as follows. 
\begin{theorem}
The optimal strategy of the target client $c_i$ in the two-stage Stackelberg game is:
\begin{equation*}
\begin{split}
\epsilon^*=
\begin{cases}
\sqrt{\frac{p-r+|\frac{pr}{s}-prs|}{r-ps}}, &
\frac{\partial^2 U_c}{\partial \epsilon^2}\leq 0,\\
\arg \max_{\epsilon\in\{\epsilon_{min},\frac{\eta^2}{2}\}} U_c(\epsilon), &\text{o.w.}\\
\end{cases}
\end{split}
\end{equation*}
\end{theorem}

As for Stage 1, the server would like to maximize its utility based on the best response from the target client, forming an optimization problem as follows:
\begin{equation*}\label{opt_2}
\begin{split}
\textbf{Problem 2:}~&p^*=\arg\max U_s(p),\\
&\,s.t.
\begin{cases}
0< p\leq p_{max},\\
U_c\geq0.\\
\end{cases}
\end{split}
\end{equation*}
Similarly, we can draw the best response for the server.
\begin{theorem}
The optimal strategy of the server in the two-stage Stackelberg game is given by:
\begin{equation*}
\begin{split}
p^*=
\begin{cases}
\frac{br-r}{H}+ \sqrt{\frac{J(H+1-b)}{s^2H^2}}, &\frac{\partial^2 U_s}{\partial p^2}\leq 0,\\
\arg \max_{p\in\{0,p_{max}\}} U_s(p), &\text{o.w.},\\
\end{cases}
\end{split}
\end{equation*}
where $H=b-s+|\frac{r}{s}-rs|, J=(ra\ln{(|D_{-i}|^2)})H+(r-br)(sa\ln{(|D_{-i}|^2)})$.
\end{theorem}
With $p^*$ in Stage 1 for the server and $\epsilon^*$ in Stage 2 for the target client, we could conclude that the Nash equilibrium point in the two-stage Stackelberg game is $(p^*,\epsilon^*)$. Based on these optimal unlearning strategies, the proposed FUI can function efficiently and consistently.



\section{Theoretical Analysis} \label{sec:theoretical}
In this section, we first theoretically demonstrate that FUI satisfies the rigorous $\epsilon$-indistinguishability during the unlearning phase, and then prove the convergence of DPFL with FUI.
\subsection{Privacy Analysis}

\begin{theorem}
Given $\eta$-DPFL, 
if the noise gap $g = \tilde{\sigma_1} - \tilde{\sigma_2} < 0$, where $\tilde{\sigma_1}$ and $\tilde{\sigma_2}$ are the noise scales required for achieving $\frac{\eta^2}{2}$-indistinguishability and $\epsilon$-indistinguishability, respectively, then adding calibrated noise $n_{cali} \sim \mathcal{G}(0, \sigma_{cali}^2)$ with $\sigma_{cali} = \sqrt{\tilde{\sigma_2}^2 - \tilde{\sigma_1}^2}$ to $w^{UN}$ ensures that the unlearning model satisfies the $\epsilon$-indistinguishability requirement from the target client. 
\end{theorem} 
\begin{proof}
According to Theorem 1, if a randomized mechanism $\mathcal{M}$ satisfies $\eta$-DP, then its output random variables $X$ and $Y$ are $\frac{\eta^2}{2}$-indistinguishable. Based on the Gaussian mechanism, we have: $\Pr[{\mathcal{M}(X) = z}] \leq e^{\frac{\eta^2}{2}} \cdot \Pr[{\mathcal{M}(Y) = z}]$, where $\mathcal{M}(X)$ and $\mathcal{M}(Y)$ follow the Gaussian distribution $\mathcal{G}(0, \tilde{\sigma_1}^2)$. Therefore, we can obtain: $\frac{e^{-\frac{z^2}{2\tilde{\sigma_1}^2}}}{\sqrt{2\pi}\tilde{\sigma_1}} \leq e^{\frac{\eta^2}{2}} \cdot \frac{e^{-\frac{(z-d)^2}{2\tilde{\sigma}_1^2}}}{\sqrt{2\pi}\tilde{\sigma}_1}.$
By solving this inequality, we can get
$\tilde{\sigma_1} \geq \frac{\sqrt{2}d}{2\eta}$.

By adding additional noise with scale: $\sigma_{cali} = \sqrt{\tilde{\sigma_2}^2 - \tilde{\sigma_1}^2}$ into noise with scale $\tilde{\sigma_1}$, we now obtain the equivalent noise scale as $\tilde{\sigma_1}^2+\sigma^2_{cali}=\tilde{\sigma_2}^2$.

Since $\tilde{\sigma_2} \geq \frac{d}{\sqrt{\epsilon}}$ we can get:
$\Pr[{\mathcal{M}(X) = z}] \leq e^{\epsilon} \Pr[{\mathcal{M}(Y) = z}]$. Similarly, we could have $\Pr[{\mathcal{M}(Y) = z}] \leq e^{\epsilon} \Pr[{\mathcal{M}(X) = z}]$.
So the unlearning model satisfies $\epsilon$-indistinguishability.
\end{proof} 


\subsection{Convergence Analysis}
\begin{theorem}
For FUI running in $\eta$-DPFL, let $w^t$ be the global model at $t$-th aggregation round, and $w_i^r$ be the local model of client $i$ at iteration $r$. Assume that the loss function $F_i(w)$ for each client $i$ is $\mu$-strongly convex and $\Upsilon$-smooth, and the variance of the stochastic gradients is bounded by $\sigma^2$. If the learning rate satisfies $\gamma_t = \frac{1}{\mu(t+1)}$ and the number of local updates per round is $K$, then the expected optimality gap after $t$ rounds is bounded by:
$\mathbb{E}[F(w^t) - F(w^*)] \leq \frac{\Upsilon}{2\mu(t+1)} \left(\frac{1}{K}\sum_{i=1}^N \| w_i^0 - w^* \|^2 + \frac{2\sigma^2}{\mu^2 K}\right)$,
where $w^*$ is the optimal solution, and $F(w) = \frac{1}{N} \sum_{i=1}^N F_i(w)$ is the global objective function.
\end{theorem}
\begin{proof}
First, we define the global objective function as $F(w) = \frac{1}{N} \sum_{i=1}^N F_i(w)$. Then, using the $\mu$-strong convexity and $\Upsilon$-smoothness of $F_i(w)$, we can bound the difference between the local and global models: $\|w_i^{r+1} - w^*\|^2 \leq (1 - \mu\gamma_t)^K \|w^t - w^*\|^2 + \frac{2\gamma_t^2\Upsilon}{\mu} \sum_{k=0}^{K-1} (1 - \mu\gamma_t)^{2k} \|g_i^{k,t}\|^2$, where $g_i^{k,t}$ is the stochastic gradient for client $i$ at local iteration $k$ of the global round $t$. Next, we bound the variance of the stochastic gradients:
$\mathbb{E}[\|g_i^{k,t}\|^2] \leq \sigma^2 + 2\Upsilon(F_i(w_i^{k,t}) - F_i(w^*))$.

Combining the bounds from the previous two steps and telescoping the inequality over $T$ rounds, we obtain: $\frac{1}{N}\sum_{i=1}^N \mathbb{E}[F_i(w_i^T) - F_i(w^*)] \leq \frac{\Upsilon}{2\mu(T+1)} \left(\frac{1}{K}\sum_{i=1}^N \| w_i^0 - w^* \|^2 + \frac{2\sigma^2}{\mu^2 K}\right)$.

Using Jensen's inequality, we can bound the global objective function: $\mathbb{E}[F(w^T) - F(w^*)] \leq \frac{1}{N}\sum_{i=1}^N \mathbb{E}[F_i(w_i^T) - F_i(w^*)]$.

Finally, combining the results from the last two steps, we obtain the final bound.    
\end{proof}

This theorem shows that under certain assumptions, the federated learning process with FUI functioning for unlearning converges to the optimal solution at a rate of $\mathcal{O}(\frac{1}{T})$, where $T$ is the number of communication rounds. The bound depends on the strong convexity and smoothness of the local objective functions, the variance of the stochastic gradients, and the number of local updates per round.


\section{Experimental Evaluation}\label{sec:experimental}
In this section, we evaluate our proposed FUI scheme and the corresponding unlearning strategies in the Stackelberg game through multiple experiments. We first introduce the experimental setup, datasets, and benchmarks. Then, we measure important performance metrics, including accuracy, prediction loss, runtime, and unlearning effectiveness (measured by Membership Inference Attack (MIA) \cite{shokri2017membership} precision and recall). Additionally, we discuss the impact of privacy parameters on FUI performance. Lastly, we demonstrate the superiority of the derived optimal unlearning strategies.
\subsection{Experiment Settings}
\subsubsection{Environment}
We evaluate our proposed FUI mechanism using Python 3.0 on Google Colab with NVIDIA Tesla V100 GPU and 16G RAM. We run the Convolutional Neural Network (CNN) model on MNIST, Purchase, Adult, and CIFAR-10 datasets. 
Each CNN comprises two convolutional layers activated by ReLU, trained using mini-batch Stochastic Gradient Descent (SGD) with a cross-entropy loss function. 
\subsubsection{Datasets}
Four datasets are detailed below, where the data is evenly distributed among clients.
\begin{itemize}[leftmargin=*]
    \item \textbf{Adult} \cite{misc_adult_2} is derived from the 1994 US Census database, including 32,561 samples for training and 16,281 samples for testing, with each detailing attributes such as age, occupation, education, and income, to predict whether an individual earns over \$50,000 annually.
    \item \textbf{Purchase} \cite{acquire-valued-shoppers-challenge} contains shopping histories from thousands of customers, including product name, categories, store chain, and date of purchase. Purchase
    dataset has 197,324 samples without any class labels. So we adopt an unsupervised algorithm \cite{shokri2017membership} to cluster the samples into two classes.
    \item \textbf{MNIST} \cite{lecun1998gradient} has 60,000 data samples for training and 10,000 data samples for testing. The samples are 28$\times$28 black and white images of handwritten digits from 0 to 9. 
    \item \textbf{CIFAR-10} \cite{alex2009learning} consists of 60,000 32$\times$32 color images in 10 classes, with 6,000 images per class. There are 50,000 training images and 10,000 testing images. 
\end{itemize}
\subsubsection{Benchmark}
Here we compare FUI with three state-of-the-art unlearning algorithms discussed in Section \ref{sec:related}: FedEraser \cite{liu2021federaser}, FedRecovery \cite{zhang2023fedrecovery}, and PGD \cite{halimi2022federated}. 
We set the number of clients $N=10$, the learning rate is 0.001, and the batch size is 100. Regarding privacy parameters, we set the default values as $\eta=5, \epsilon=5$, unless specified otherwise. 
It is worth noting that since other FU schemes do not include DP, we have included the same scale of noise in these schemes to achieve the same level of DP protection using the method of~\cite{wei2020federated} for a fair comparison. In addition, we define a baseline that implements the exact retraining process in DPFL involving all clients, denoted as the ``Retrain'' method.

In experiments about the Stackelberg game, we set the default value of parameters as $a=1.5, b=10, k=25, s=2, \epsilon_{min}=0.1, p_{max}=15, \eta=5, \Psi_s=5, \Psi_c=3$ unless specified otherwise.

\subsection{Experimental Results}
\subsubsection{Accuracy}
\begin{table}[!t]
\renewcommand{\arraystretch}{1.3}
\caption{Comparison of accuracy among different FU methods.}
\label{tab:acc}
\centering
\begin{tabular}{c|cccc}
\hline
\multirow{2}{*}{Method} & \multicolumn{4}{c}{Accuracy (\%)} \\
\cline{2-5}
& Adult & Purchase & MNIST & CIFAR-10 \\
\hline
Retrain & 73.2
& 83.9
& 86.4
& 60.5
\\
FedEraser \cite{liu2021federaser} & 66.8
& 76.3
& 81.2
& 55.2
\\
FedRecovery \cite{zhang2023fedrecovery} & 65.1 
& 75.7 
& 78.5
& 54.8
\\
PGD \cite{halimi2022federated} & 66.4 
& 74.3 
& 80.6
& 52.1
\\
Our method & 67.3
& 75.8
& 81.0
& 58.9
\\
\hline
\end{tabular}
\end{table}
The accuracy of the unlearning model reflects the general performance of the unlearning algorithm from the perspective of the server. Table \ref{tab:acc} presents the comparison of model accuracy in different unlearning methods on four datasets. We could observe that the Retrain method reaches the highest accuracy over all four datasets. The accuracy of our proposed FUI scheme is higher than the other three schemes in most cases, which is due to the fact that we take advantage of the properties of the DP mechanism to reach the unlearning effect through local model retraction and global noise calibration, and thus is more adapted to the DPFL scenario. 

More specifically, on Adult, we are 5.9\% away from the Retrain method, 0.5\% ahead of FedEraser, 2.2\% ahead of FedRecovery and 0.9\% higher than PGD; 
on Purchase, our accuracy is competitive, with 0.1\% and 1.5\% higher than FedRecovery and PGD respectively; on MNIST, we are 5.4\% away from the Retrain method, 0.2\% lower than FedEraser, 2.5\% more accurate than FedRecovery, and 0.4\% higher than PGD;
and on CIFAR-10, there is only a difference of 1.6\% between FUI and Retrain methods. The accuracy of FedEraser is also close to Retrain's because it accelerates the unlearning process through optimizations based on retraining and fine-tuning, while PGD's accuracy is slightly lower because it is more susceptible to Gaussian noise. 
The accuracy of FedRecovery inevitably suffers due to the injection of additional Gaussian noise into the unlearning model without confirming the necessity.
\subsubsection{Prediction Loss}
\begin{table}[!t]
\renewcommand{\arraystretch}{1.3}
\caption{Comparison of prediction loss among different FU methods.}
\label{tab:loss}
\centering
\begin{tabular}{c|cccc}
\hline
\multirow{2}{*}{Method} & \multicolumn{4}{c}{Prediction Loss} \\
\cline{2-5}
& Adult & Purchase & MNIST & CIFAR-10 \\
\hline
Retrain & 6.87×$10^3$
& 5.88×$10^3$
& 3.18×$10^3$
& 1.86×$10^5$
\\
FedEraser \cite{liu2021federaser} & 8.77×$10^3$
& 3.29×$10^4$
& 2.25×$10^4$
& 5.26×$10^6$
\\
FedRecovery \cite{zhang2023fedrecovery} & 9.01×$10^4$
& 2.37×$10^5$
& 5.87×$10^3$
& 1.02×$10^7$
\\
PGD \cite{halimi2022federated} & 5.92×$10^5$ 
& 7.64×$10^5$ 
& 1.98×$10^4$
& 2.83×$10^7$
\\
Our method & 4.51×$10^3$
& 6.13×$10^4$
& 6.57×$10^4$
& 9.49×$10^5$
\\
\hline
\end{tabular}
\end{table}
Our next experiment involves the comparison of prediction loss between our proposed method and existing solutions. Lower values of prediction loss indicate better performance of the unlearning model. Table \ref{tab:loss} shows that Retrain has the lowest prediction loss, which corresponds to the high accuracy it exhibits in Table \ref{tab:acc}. Meanwhile, FedEraser has quite good performance on the Adult and Purchase datasets, which is very close to Retrain's prediction loss. However, on the MNIST and CIFAR-10 datasets, our proposed FUI scheme is superior. This may be related to the type of dataset, since the input data of Adult and Purchase are numerical, while the MNIST and CIFAR-10 datasets are classical image datasets. The FedRecovery algorithm also performs well on the MNIST and Adult datasets; however, PGD has a high prediction loss on all of them.
\subsubsection{Time Consumption}
\begin{table}[!t]
\renewcommand{\arraystretch}{1.3}
\caption{Comparison of running time among different FU methods.}
\label{tab:time}
\centering
\begin{tabular}{c|cccc}
\hline
\multirow{2}{*}{Method} & \multicolumn{4}{c}{Time Consumption (second)} \\
\cline{2-5}
& Adult & Purchase & MNIST  & CIFAR-10 \\
\hline
Retrain & 268.7
& 129.4
& 482.3
& 1863.6
\\
FedEraser \cite{liu2021federaser} & 92.9
& 57.9
& 136.4
& 819.2
\\
FedRecovery \cite{zhang2023fedrecovery} & 1.3
& 1.6 
& 1.7
& 2.1
\\
PGD \cite{halimi2022federated} & 98.2 
& 81.6
& 121.5 
& 762.0 
\\
Our method & 5.8
& 4.7
& 4.6
& 6.3
\\
\hline
\end{tabular}
\end{table}
The running time is also a critical performance metric for unlearning methods, indicating their efficiency. FUI demonstrates a significant reduction in time consumption across all datasets, which proved to be an efficient unlearning process. Retrain is undoubtedly the most time-consuming approach. The FedEraser algorithm also takes a significant amount of time due to the need for multi-client retraining, although it is already accelerated compared to Retrain. The PGD algorithm also takes more time due to the post-training process. FedRecovery is the fastest unlearning algorithm through its feature of only adding Gaussian noise to achieve unlearning. After comparison, we can see that although FUI is not the fastest unlearning algorithm, its efficiency is highly acceptable.

\subsubsection{Unlearning Performance}
\begin{figure}
  \centering
  \subfigure[Precision]{
    \label{fig:subfig:2a}
    \includegraphics[scale=0.26]{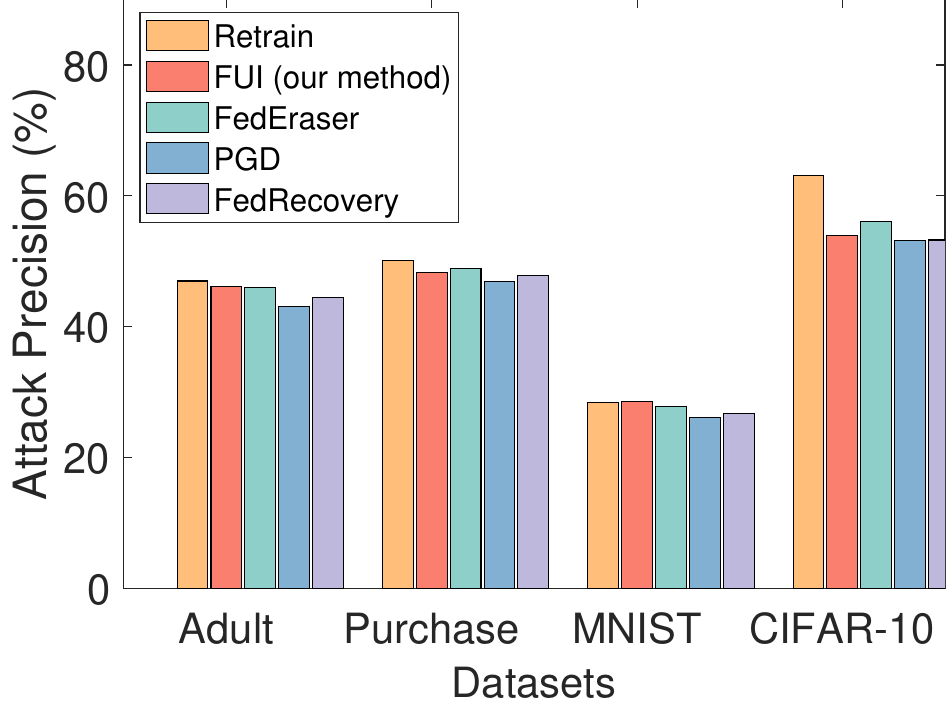}}
  \subfigure[Recall]{
    \label{fig:subfig:2b}
    \includegraphics[scale=0.26]{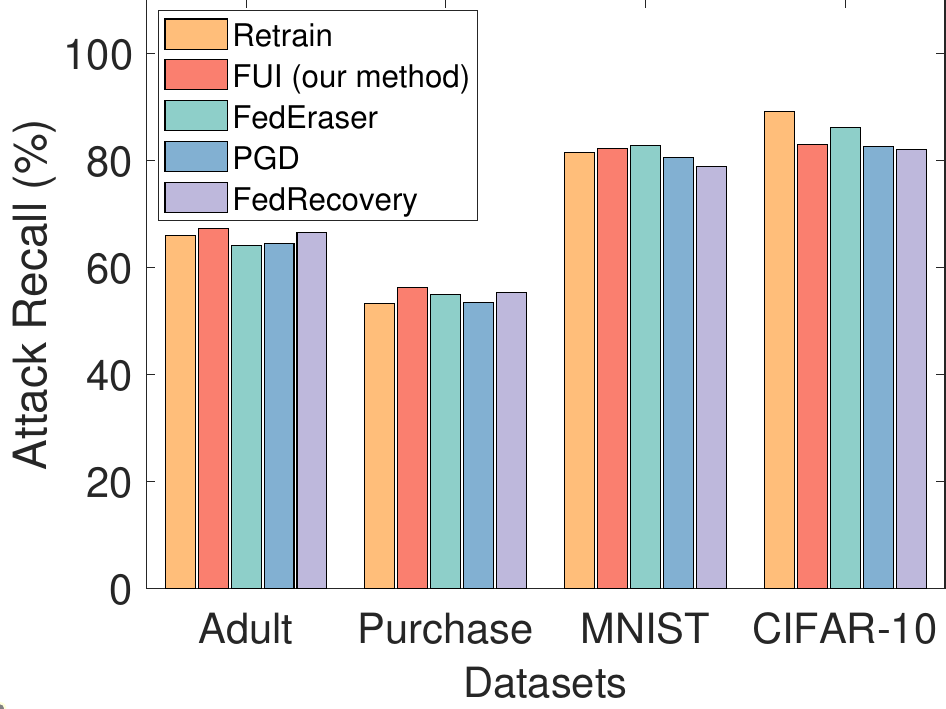}}
  \caption{Comparison of MIA performance among different FU methods.}
  \label{fig:MIA}
\end{figure}
The main purpose of federated unlearning is to eliminate the influence of client data, keeping as little information from the target client as possible in the unlearning model. In this experiment, we evaluate how much information about the target client remains in the unlearning model by performing MIA when the unlearning model is running on different datasets. To perform MIA, we first train shadow models on a dataset similar to the target model’s dataset to simulate its behavior. Then, we use these shadow models to create an attack model to distinguish between data that was part of the training dataset and data that was not. 

Figs. \ref{fig:MIA}(a)(b) respectively show the precision and recall of MIA performed on the unlearning models derived from Retrain and benchmark FU methods, as well as our proposed FUI. 
We could observe that the precision and recall of all five schemes are very close in the first three datasets, i.e., Adult, Purchase, and MNIST. 
In fact, the close performance of unlearning models against MIA indicates the similarity of information retention of the target client in the unlearning models. These results show that, like the retrained model, the unlearning model of FUI contains little information about the target client after effectively unlearning the target client in DPFL. When we compare FUI with three existing FU methods, FedEraser is very comparable to FUI and performs a bit better on CIFAR-10, while FUI exhibits better performance than FedRecovery and PGD most of the time.
In particular, Fig.~\ref{fig:MIA}(a) shows that FUI achieves closer precision to the retrained model than FedRecovery and PGD while maintaining similar precision as FedEraser on all datasets. Fig. \ref{fig:MIA}(b) highlights that FUI outperforms similar to three benchmark solutions. 
\subsubsection{Impact of Privacy Parameters}
\begin{figure}[htbp]
  \centering
  \begin{minipage}[b]{0.32\textwidth}
    \centering
    \includegraphics[width=\textwidth]{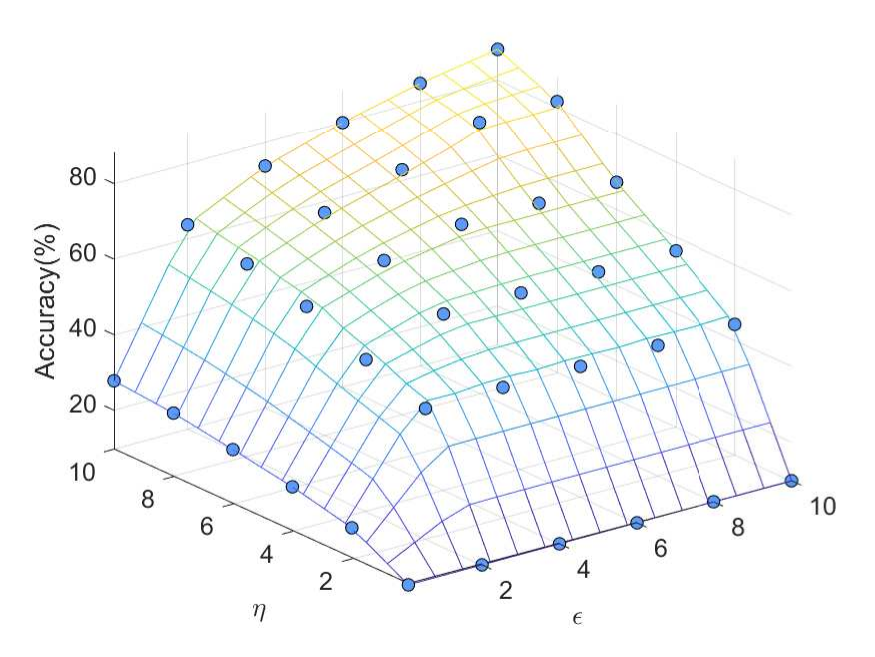}
    \caption*{(a)}
    \label{fig:subfig:3d}
  \end{minipage}
  \begin{minipage}[b]{0.16\textwidth}
    \centering
    \includegraphics[width=\textwidth]{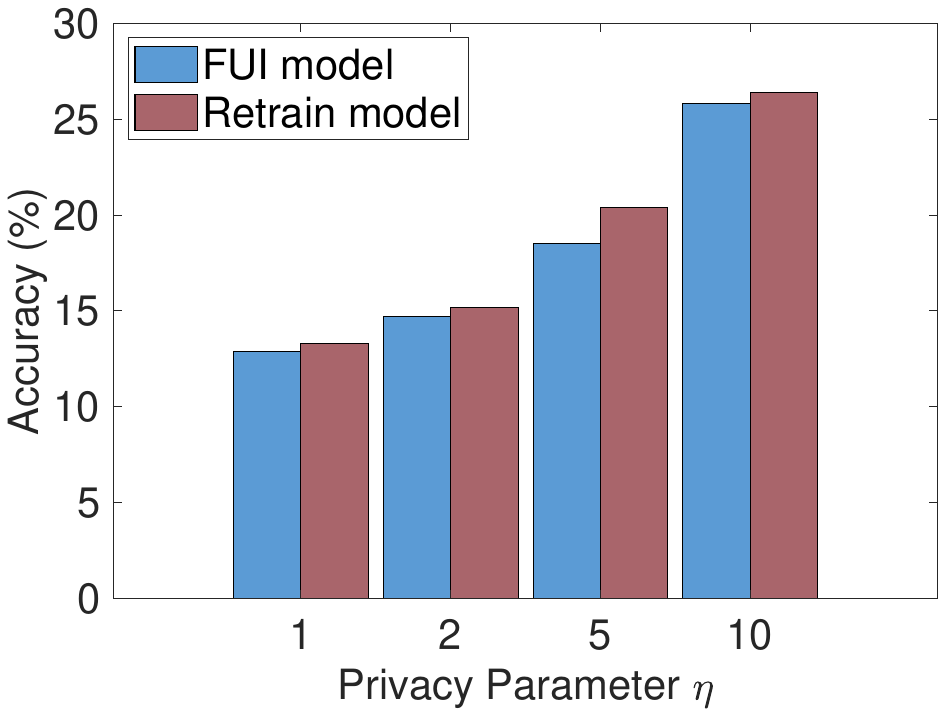}
    \caption*{(b)}
    \label{fig:subfig:3c}
    \vspace{4pt} 
    \includegraphics[width=\textwidth]{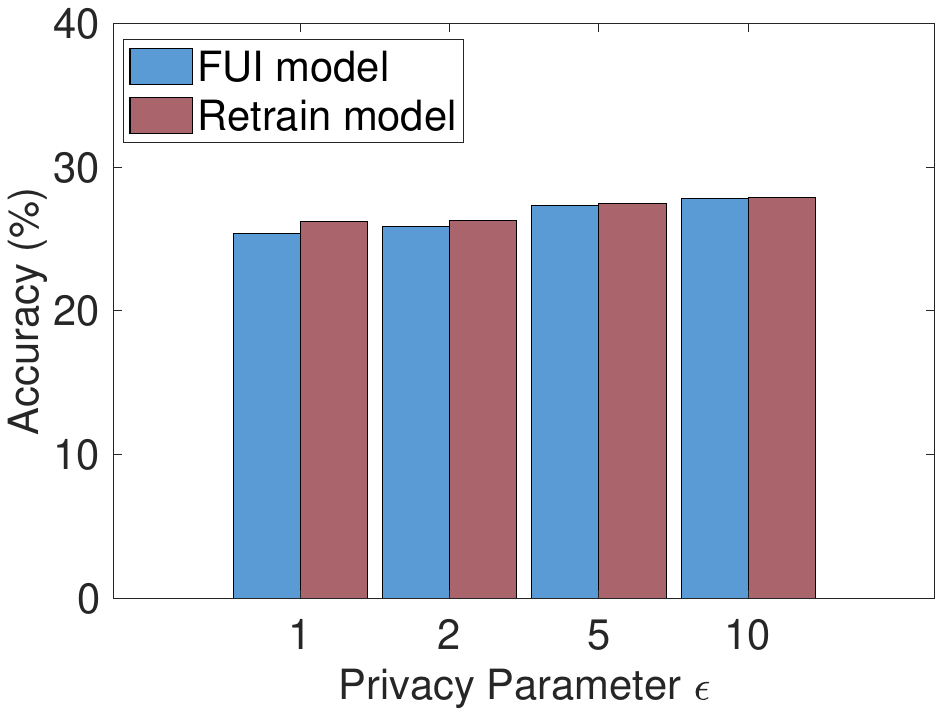}
    \caption*{(c)}
    \label{fig:subfig:3d2}
  \end{minipage}
  \caption{Evolution of accuracy under different privacy parameters.}
  \label{fig:privacy}
\end{figure}

In DPFL scenarios, privacy-preserving parameters are crucial. FUI has two privacy parameters, i.e., $\eta$ and $\epsilon$. Here we investigate the effect of privacy parameters on our proposed FUI through experiments and present the results in Fig. \ref{fig:privacy}.

Fig. \ref{fig:privacy}(a) shows comprehensively how $\eta$ and $\epsilon$ affect the accuracy of the unlearning model running on CIFAR-10. 
Note that changing trends on other datasets are similar, so here we take the result on CIFAR-10 as an example. 
Fig. \ref{fig:privacy}(b) shows the accuracy of the FUI and retrained model corresponding to different $\eta$ given $\epsilon=1$. And Fig. \ref{fig:privacy}(c) shows how their accuracy varies with $\epsilon$ given $\eta=1$.
We can observe that accuracy generally rises with $\eta$ or $\epsilon$, and the increase rate slows down gradually. This shows the typical effect of privacy parameters on model accuracy in DPFL. It is worth noting that when $\eta$ is small, accuracy no longer changes after $\epsilon$ reaches a certain value. This is because a smaller $\eta$ implies a higher level of DP protection, and the indistinguishability the unlearning model achieves via local model retraction is sufficient to meet the requirement of unlearning corresponding to $\epsilon$, and thus increasing $\epsilon$ will not further impact the accuracy. In addition, we can find that the unlearning model accuracy of FUI and the retrained model have similar performance. This further validates the effectiveness of our scheme.

\subsubsection{Evaluation of Stackelberg Game}
To verify the effectiveness of the derived optimal unlearning strategies for FUI, we conduct a series of simulation experiments. We explore the impacts of strategy pairs on the utilities of the server and target client changing with different parameters. For simplicity, we denote the server's optimal strategy as SO and the server's random strategy as SR; similarly, CO and CR represent the client's optimal strategy and random strategy, respectively.

Fig. \ref{fig:FUI_s} illustrates the impact of various parameters on the server's utility under four strategy combinations: SO vs. CO (blue solid line), SO vs. SR (blue dashed line), SR vs. CO (red solid line), and SR vs. CR (red dashed line). The results demonstrate that SO always yields higher utility compared to SR across different parameter settings, confirming the effectiveness of the optimal strategy in maximizing the server's utility. Moreover, when the server adopts the optimal strategy, $|D_{-i}|$ and $a$ are positively correlated with the server's utility since they control the loss of performance for the unlearning model. Particularly, $a$ has the most significant proportional effect on its utility. Parameter $b$, on the other hand, is negatively correlated with the server's utility. While parameters $r,s,$ and $p_{max}$ have almost no effect on the server's utility.

Fig. \ref{fig:FUI_c} presents the effect of varying parameters on the client's utility under the same four strategy combinations as in Fig.~\ref{fig:FUI_s}. The results indicate that CO always produces higher utility than CR across different situations. It is worth noting that the client's utility is always greater than 0 in all cases because the client would only request unlearning when its utility is not negative. $r$ is proportional to the client's utility under all combinations of strategies. And $s$ is negatively correlated with the client's utility since its increase would decrease the benefit of unlearning. Parameters $|D_{-i}|,a,b,$ and $p_{max}$ have no obvious effect on the client's utility.

\begin{figure}
  \centering
  \subfigure{
    \label{fig:subfig:4a}
    \includegraphics[scale=0.17]{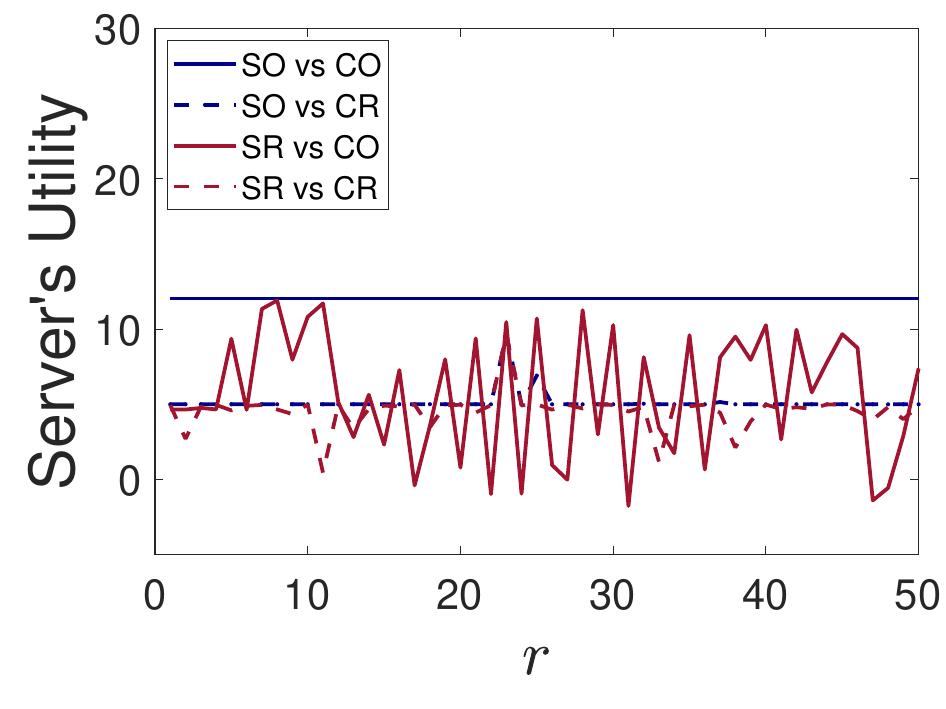}}
  \subfigure{
    \label{fig:subfig:4b}
    \includegraphics[scale=0.17]{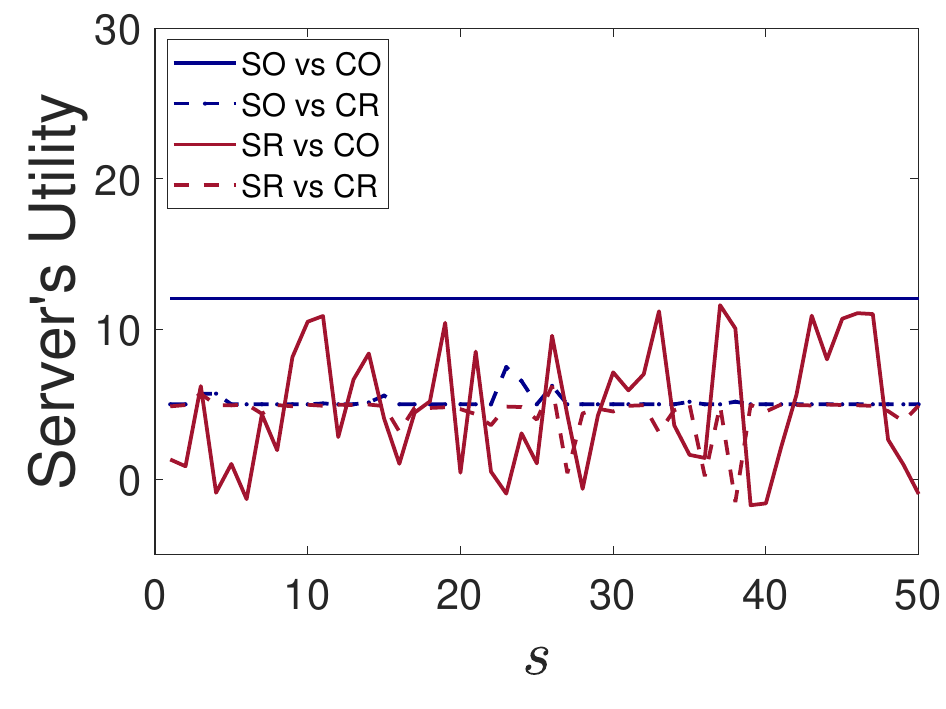}}
  \subfigure{
    \label{fig:subfig:4c}
    \includegraphics[scale=0.17]{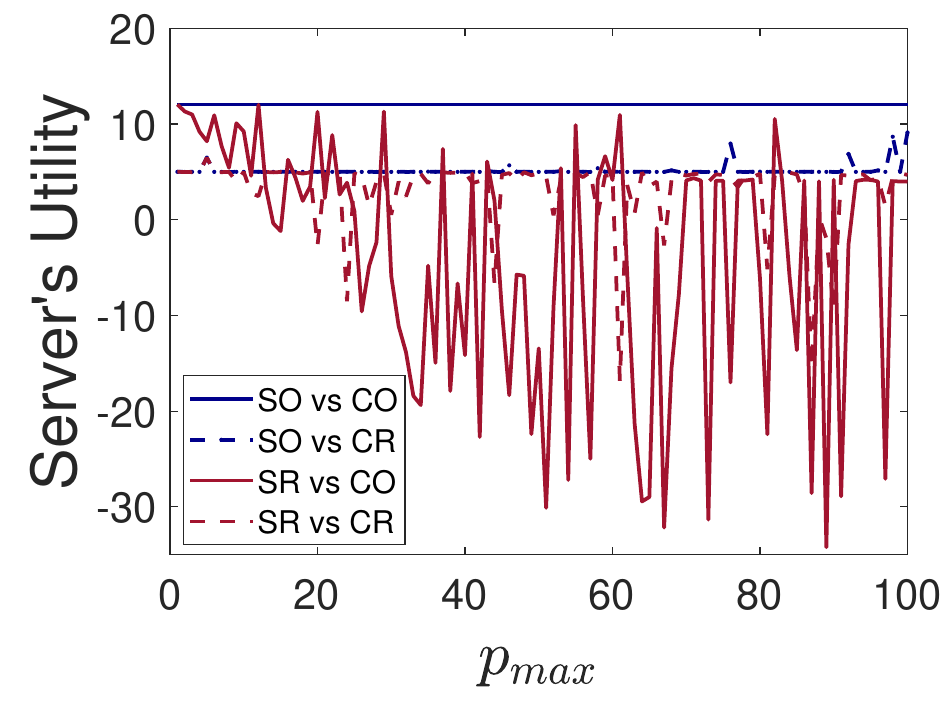}}
  \subfigure{
    \label{fig:subfig:4d}
    \includegraphics[scale=0.17]{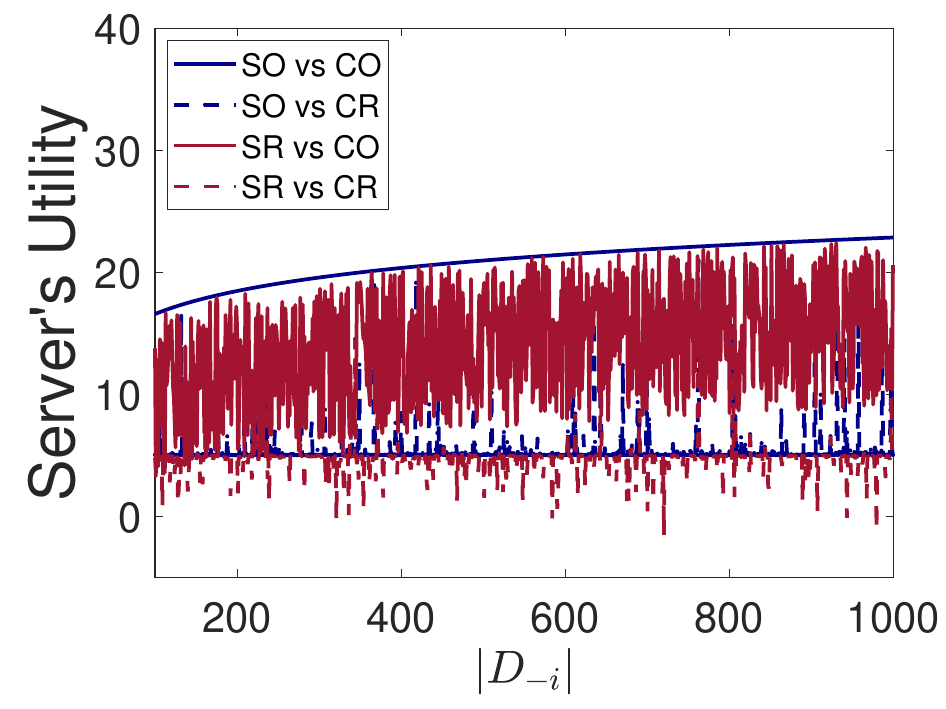}}
  \subfigure{
    \label{fig:subfig:4e}
    \includegraphics[scale=0.17]{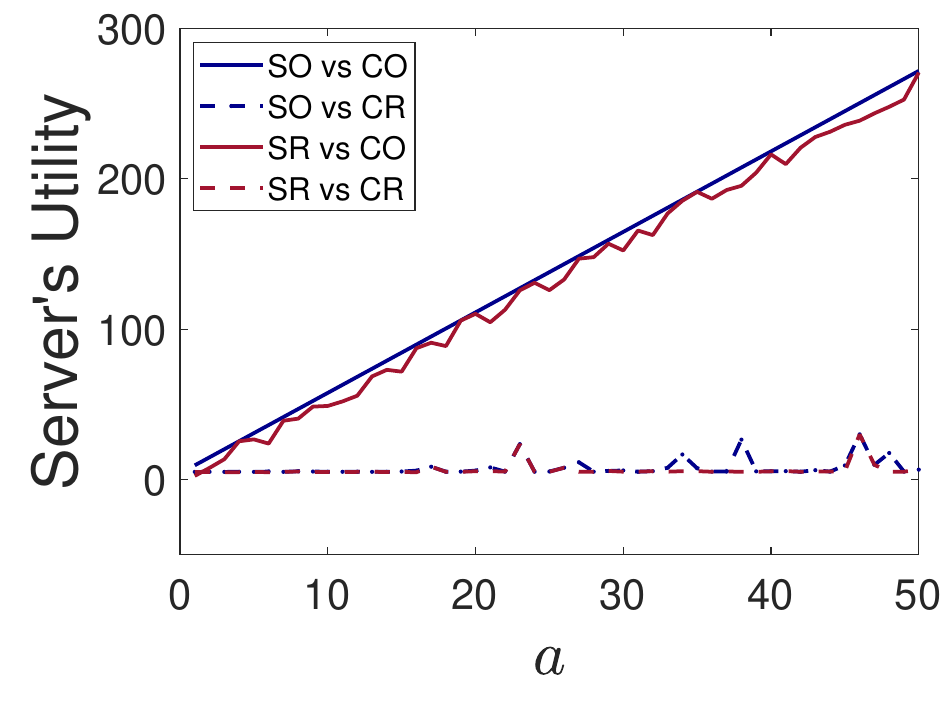}}
  \subfigure{
    \label{fig:subfig:4f}
    \includegraphics[scale=0.17]{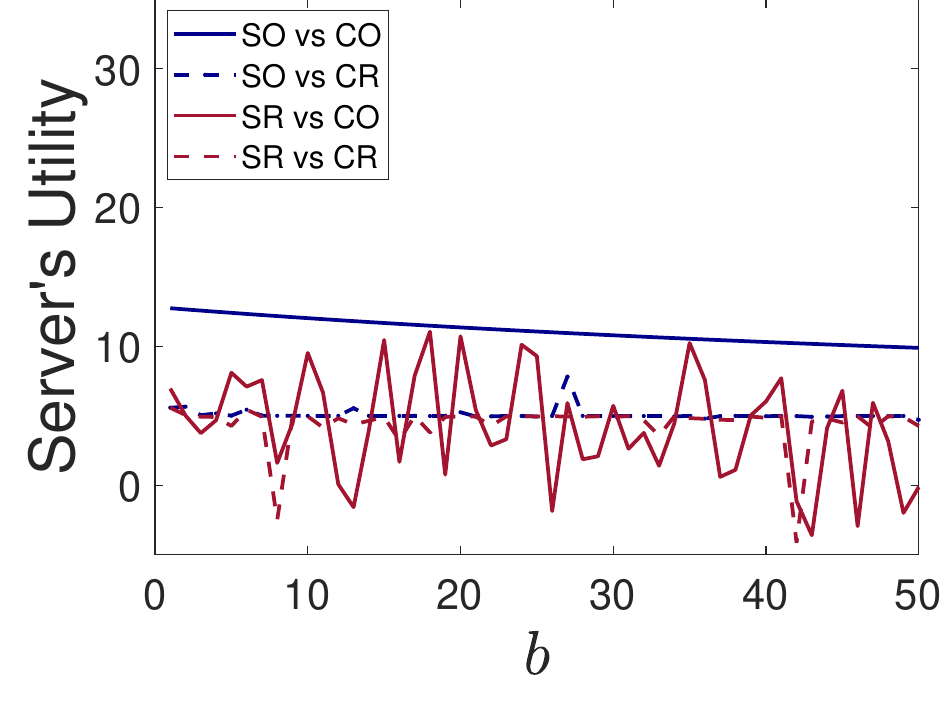}}
  \caption{Impact of strategies on server's utility with different parameters.}
  \label{fig:FUI_s}
\end{figure}

\begin{figure}
  \centering
  \subfigure{
    \label{fig:subfig:5a}
    \includegraphics[scale=0.17]{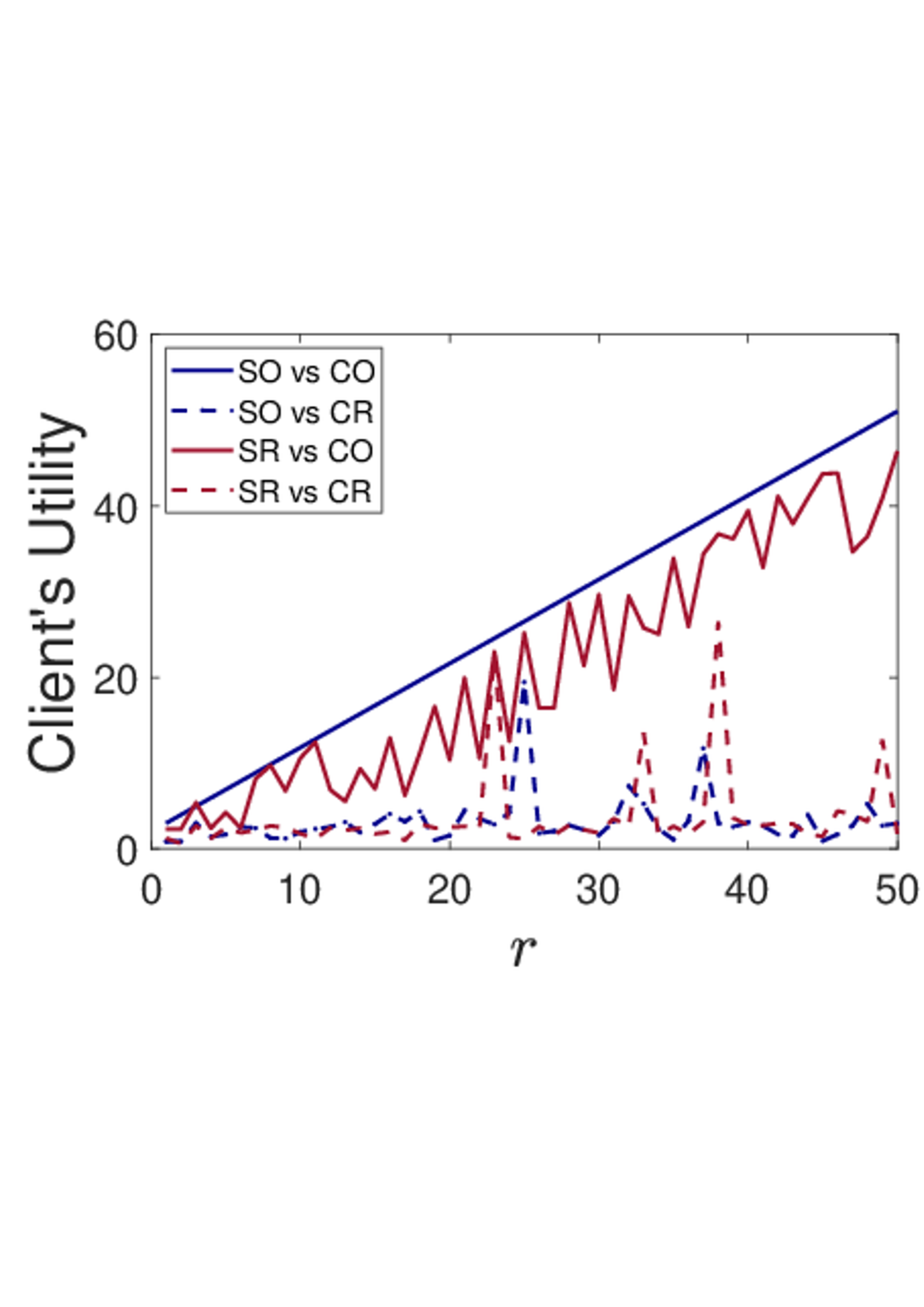}}
  \subfigure{
    \label{fig:subfig:5b}
    \includegraphics[scale=0.17]{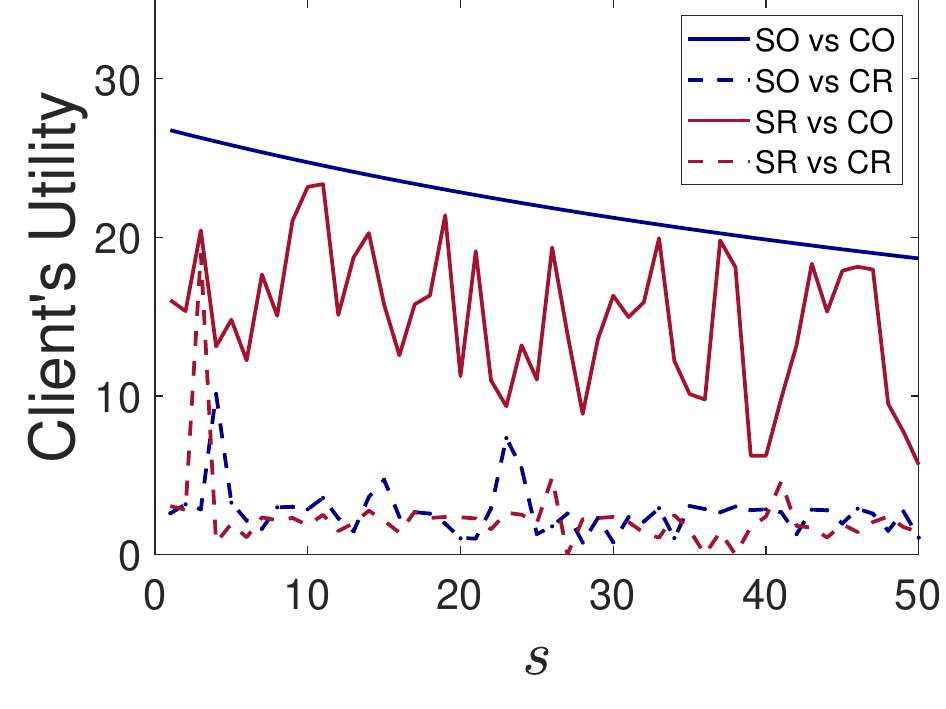}}
  \subfigure{
    \label{fig:subfig:5c}
    \includegraphics[scale=0.17]{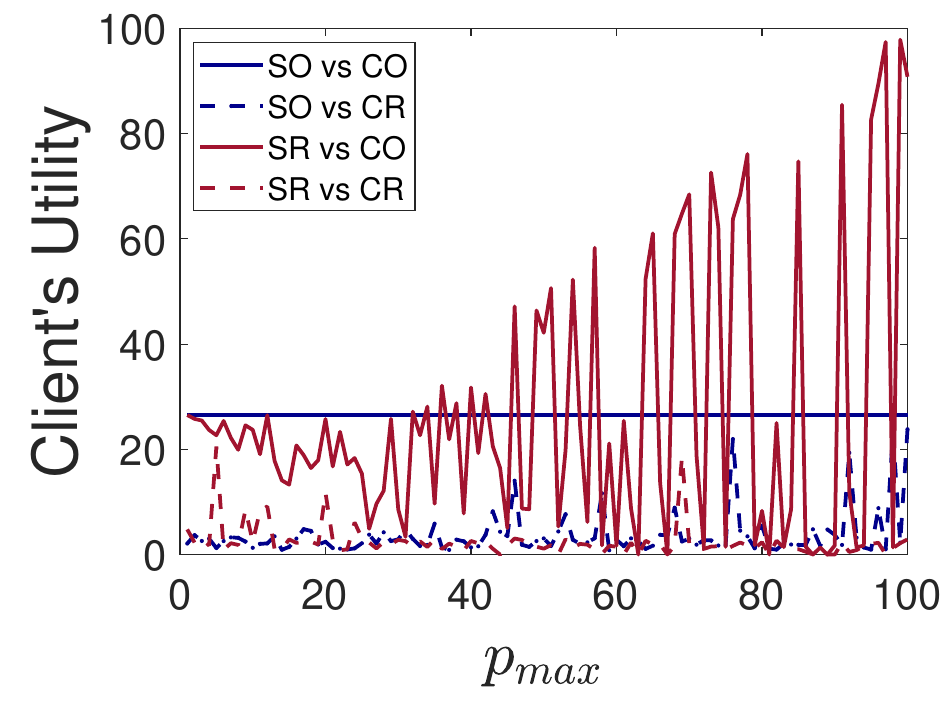}}
  \subfigure{
    \label{fig:subfig:5d}
    \includegraphics[scale=0.17]{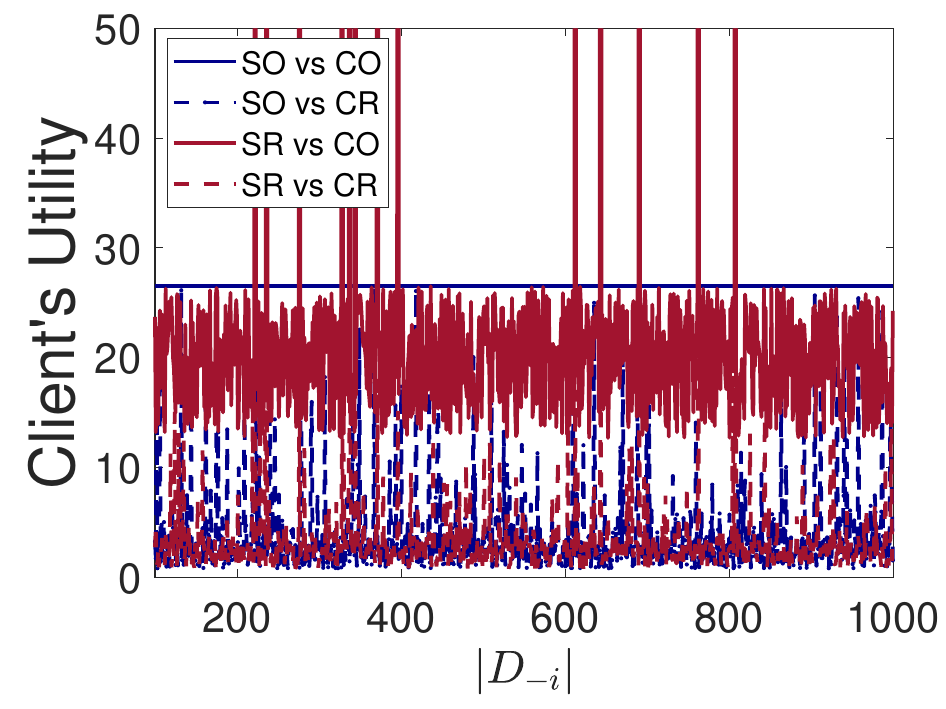}}
  \subfigure{
    \label{fig:subfig:5e}
    \includegraphics[scale=0.17]{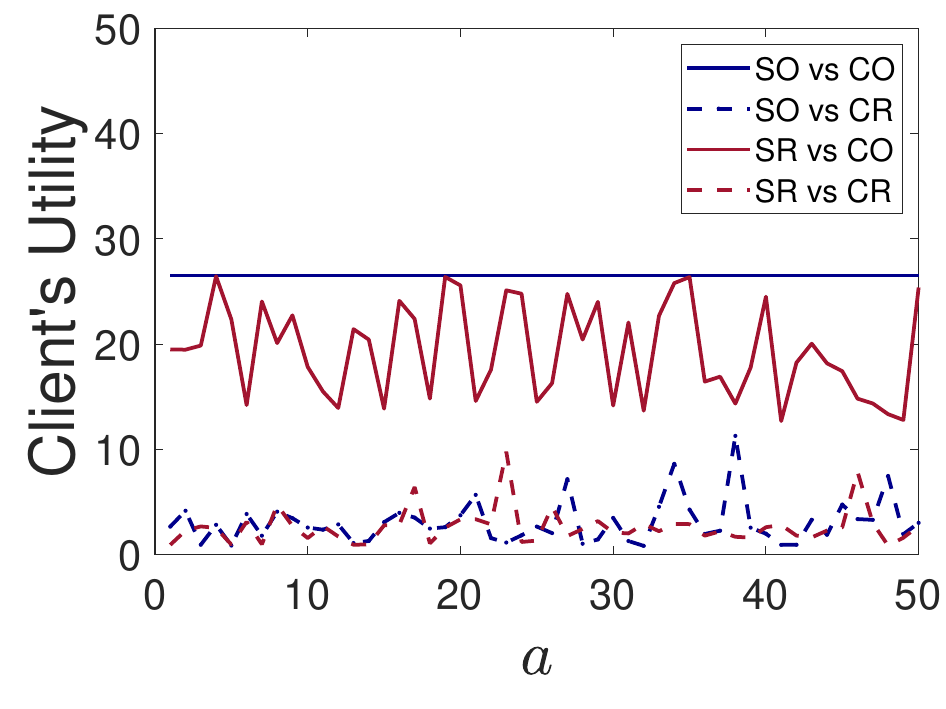}}
  \subfigure{
    \label{fig:subfig:5f}
    \includegraphics[scale=0.17]{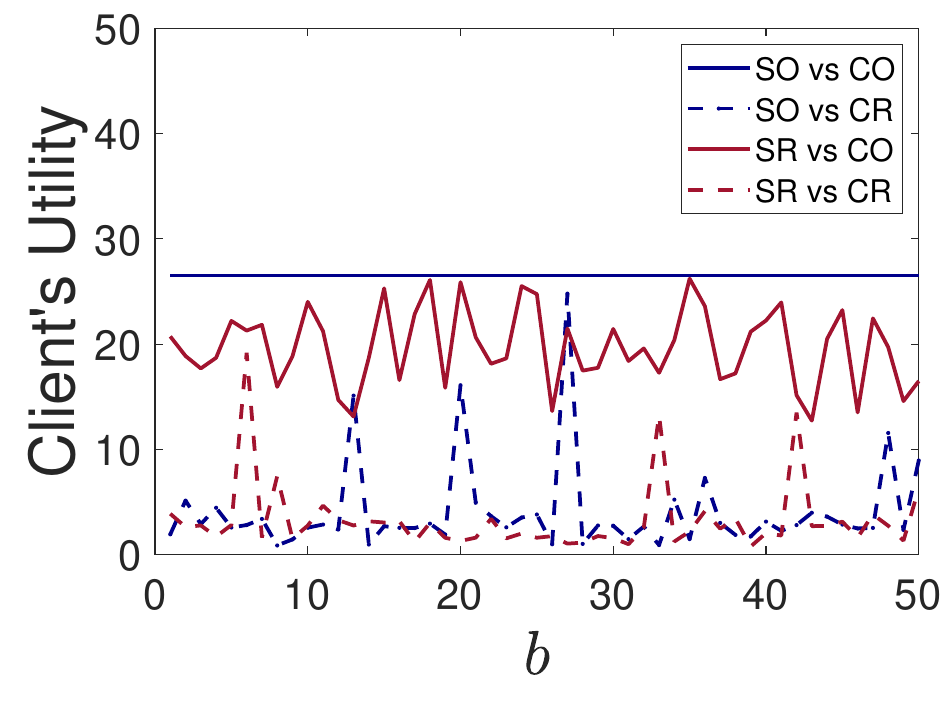}}
  \caption{Impact of strategies on client's utility with different parameters.}
  \label{fig:FUI_c}
\end{figure}

\section{Conclusion and Future Work}\label{sec:conclusion}
In this paper, we propose FUI to fill the research gap of missing unlearning schemes for DPFL. Our FUI scheme consists of two steps: local model retraction and global noise calibration.
In local model retraction, we recycle the noise of the DP mechanism to achieve a certain level of indistinguishability while performing local model retraction. In global noise calibration, we calculate the potential noise gap and then enhance the effect of unlearning by adding noise if necessary. To facilitate the implementation of FUI, we calculate the optimal unlearning strategies for the server and the target client by formulating a two-stage Stackelberg game and resolve two optimization questions sequentially. 
We also conduct privacy analysis and convergence analysis to provide theoretical guarantees. Experiments based on real-world data show that FUI has higher efficiency and better performance than other existing schemes in DPFL; simulation results verify that the optimal strategies can lead to the best utilities for both parties. 
For future work, we will 
focus on simultaneously unlearning multiple local datasets for target clients in DPFL. In addition, we will explore other methods for verifying the unlearning effectiveness tailored for DPFL, especially considering that the noise injected into local models can impact the verification.
\newpage
\bibliographystyle{IEEEtran}
\bibliography{./reference.bib,./IEEEexample}

\end{document}